\newcommand{\removed}[1]{}

\documentclass[runningheads,a4paper]{llncs}
\usepackage{color}
\usepackage{authblk}
\usepackage{mathtools}

\usepackage{amsthm}
\usepackage{verbatim}
\usepackage{latexsym}
\usepackage{xspace}
\usepackage{algorithm}
\usepackage{algorithmic}
\usepackage[utf8]{inputenc}
\usepackage{perpage} 
\MakePerPage{footnote} 

\newtheorem{observation}[theorem]{Observation}

\newtheorem{Problem}[theorem]{Problem}

\newcommand{\NN}{\mathcal{N}}

\newcommand{\supp}{\text{supp}}

\renewcommand{\to}{\mapsto}

\renewcommand{\tt}{\mathcal{T'}}

\def\eps{\epsilon}

\def\bull{\vrule height .9ex width .8ex depth -.1ex }

\numberwithin{equation}{section}

\newcommand{\beq}{\begin{eqnarray}}
\newcommand{\eeq}{\end{eqnarray}}

\newcommand{\sign}{\text{sign}}

\usepackage{url}
\urldef{\maila}\path|bneyshabur@ttic.edu|    
\urldef{\mailb}\path|rina@microsoft.com|    
\newcommand{\keywords}[1]{\par\addvspace\baselineskip
\noindent\keywordname\enspace\ignorespaces#1}

\begin{document}

\mainmatter  

\title{Sparse Matrix Factorization}

\author{Behnam Neyshabur\inst{1} \and Rina Panigrahy\inst{2}}


%
%
%
\authorrunning{Behnam Neyshabur and Rina Panigrahy}

\institute{Toyota Technological Institute at Chicago\\
\maila
\and Microsoft Research\\
\mailb
}

%
%

\maketitle

\begin{abstract}
We investigate the problem of factoring a matrix into several sparse matrices and propose an algorithm for this under randomness and sparsity assumptions. This problem can be viewed as a simplification of the deep learning problem where finding a factorization corresponds to finding edges in different layers and also values of hidden units. We prove that under certain assumptions on a sparse linear deep network with $n$ nodes in each layer, our algorithm is able to recover the structure of the network and values of top layer hidden units for depths up to $\tilde O(n^{1/6})$. We further discuss the relation among sparse matrix factorization, deep learning, sparse recovery and dictionary learning.
\end{abstract}

\section{Introduction}

In this paper we study the following matrix factorization problem. The sparsity $\pi(X)$ of a matrix $X$ is the number of non-zero entries in $X$.

\begin{Problem}[Sparse Matrix-Factorization]
Given an input matrix $Y$ factorize it is as $Y = X_1 X_2\dots X_s$ so as minimize the total sparsity $\sum_{i=1}^s \pi(X_i)$.
\end{Problem}

The above problem is a simplification of the non-linear version of the problem that is directly related to learning using deep networks.

\begin{Problem}[Non-linear Sparse Matrix-Factorization]
Given matrix $Y$, minimize $\sum_{i=1}^s \pi (X_i) $ such that $\sigma( X_1. \sigma( X_2 . \sigma( \dots X_s))) = Y$
where $\sigma(x)$ is the sign function ($+1$ if $x>0$, $-1$ if $x<0$ and $0$ otherwise) and $\sigma$ applied on a matrix is simply applying the sign function on each entry. Here entries in $Y$ are $0,\pm 1$.
\end{Problem}

{\bf Connection to Deep Learning and Compression:}
The above problem is related to learning using deep networks (see \cite{bengio09}) that are generalizations of neural networks. They are layered network of nodes connected by edges between successive layers; each node applies a non-linear operation (usually a sigmoid or a perceptron) on the weighted combination of inputs along the edges. Given the non-linear sigmoid function and the deep layered structure, they can express any circuit. The weights of the edges in a deep network with $s$ layers may be represented by the matrices $X_1, \dots, X_s$. If we use the sign function instead of the step function, the computation in the neural network would exactly correspond to computing $Y = \sigma( X_1. \sigma( X_2 . \sigma( \dots X_s)))$. Here $X_s$ would correspond to the matrix of inputs at the top layer.

There has been a strong resurgence in the study of deep networks resulting in major breakthroughs in the field of machine learning by Hinton and others~\cite{hinton06,ranzato06,bengio06}.  Some of the best state of the art methods use deep networks for several applications including speech, handwriting, and image recognition \cite{Krizhevsky12,Goodfellow13,li13}.

Traditional neural networks were typically used for supervised learning and are trained using the gradient descent style back propagation algorithm. More recent variants have been using unsupervised learning for pre-training, where the deep network can be viewed as a generative model for the observed data $Y$. The goal then is to learn from $Y$ the network structure and the inputs that are encoded by the matrices $X_1,\dots, X_s$. In one variant called Deep Boltzmann Machines, each layer is a Restricted Boltzmann Machines (RBM) that are reversible in the sense that inputs can be produced from outputs by inverting the network \cite{salakhutdinov09}. Auto-encoders are another variant to learn deep structures in the data \cite{bengio06}. One of the main differences between auto-encoders and RBMs is that in an RBM, the weights of edges for generating the observed data is the same as recovering hidden variables, i.e. the encoding and decoding functions are the same; however, auto-encoders allow different encoder and decoders \cite{bengio13}. Some studies have shown that it is beneficial to insist on sparseness either in the number of edges or the number of active nodes in each layer \cite{bengio13,ranzato07}.
On the other side, not much is known about the theory behind deep networks and why they are able to learn much more complex patterns in the data. Recently, \cite{arora2} gave an algorithm for learning random sparse deep networks upto a certain depth -- this is basically an algorithm for non-linear sparse matrix factorization.

If we measure the complexity of a deep network by the number of its edges then the above non-linear sparse factorization problem is identical to the problem of finding the simplest deep network when each node applies the sign function instead of the sigmoid. A deep network that produces a matrix $Y$ can naturally be viewed as a compressed representation of $Y$. Thus if $Y$ is a matrix that represents some sensory input, where say each column is an image then expressing $Y$ as outputs of a deep network is equivalent to ``compressing" $Y$ which is like a simpler explanation of $Y$. The nodes in the network may represent different concepts in the images. Each column in each matrix is a concept. Since neural networks can emulate any  circuit (by and/or/not gates with at most $O(\log n)$ blow up in size -see appendix ~\ref{sec:kolcircuit}) computing the smallest network  is a cryptographically hard problem. The network with the smallest number of edges translates into the fewest non-zero entries or maximum sparsity in the Non-linear-Matrix-factorization problem.

{\bf Connection to PCA, Dictionary learning, Sparse encoding:}
In fact many known learning algorithms can be viewed as solving special cases of the sparse matrix factorization problem.
For example $PCA$ can be stated as writing $Y = X_1 X_2$ where $X_1,X_2$ are rank $d$ this is simply a special case of $d$-sparse rows (columns).

Note that a special case of a sparse matrix is a matrix with a small number of columns (or rows). Sparse encoding \cite{arora1} can be viewed as the problem of writing $Y = X_1X_2$ where $X_1$ is a dictionary that has much fewer columns than $Y$ (which
is a special case of sparse) and $X_2$ is sparse. Sparse encoding problem arises when $X_1$ is known.

Thus motivated by the connection to deep networks and compression, we will study the problem of sparse matrix factorization. Computing the smallest circuit that expresses a matrix is cryptographically hard and is in fact it is as hard as inverting a one way function (which is as hard as integer factoring -- see appendix ~\ref{sec:kolcircuit})

So rather than focusing on hard instances, we will focus on random instances when all the matrices are $d$-sparse and of order $n$.
In a recent work by~\cite{arora2},  the authors propose an algorithm for random instances of non-linear sparse matrix factorization when the depth $s$ is at most $O(\log_d n)$. Here we show that factorization can be achieved even for depths up to $\tilde O(n^{1/6})$. We also note that when $s \le \log_d n$, then most entries in the non-linear product match the entries in the linear product; this is because the expected number of non-zero entries at any node is at most $1$ in which case the $\sigma$ operator would not make a difference.

Here we will provide a simple algorithm for sparse matrix factorization for the linear case that can be interpreted as a natural algorithm for growing a deep network from the  bottom layers to the top -- our algorithm is very similar to that in \cite{arora2}. This is very different from standard approaches in constructing a layer of RBM that creates  an arbitrary bipartite graph of edges that are initialized randomly and then adjusted using gradient descent. Our algorithm on the other hand  creates a new node on the layer above as and when we find some nodes in the lower layer to be firing in a correlated fashion. The main principle for creating new nodes and edges is simple: the networks grows from bottom to top one layer at a time. For constructing each layer, first we observe correlations between all pairs of inputs in the bottom layer and then find clusters of highly correlated inputs to create a new hidden node on a layer above. Finding the cluster of correlated nodes is also done using a simple and natural process: a pair of correlated nodes are connected to a new hidden node in the layer above; then additional nodes correlated to the pair of nodes are added followed by some pruning operations.

\section{Results}
Let $Y = X_1 X_2 \dots X_s (1/\sqrt d)^s $ where $X_i$s are i.i.d random $d$-sparse matrices and $1/\sqrt d$ is used as a scaling factor so that the norm of each column becomes $1$. For simplicity in analysis, we will assume that each column of $X_i$ is a sum of $d$ random $1$-sparse column vectors (where the non-zero entry is $\pm 1$ with equal probability). We will refer to such a column vector as a random $d$-sparse vector (although it is possible that it has less than $d$ non-zero entries). We will refer to a matrix as a random $d$-sparse matrix if each column is an independent $d$-sparse vector. All the matrices $X_1, \dots, X_s$ will be produced in this way. We will assume that $Y$ is known up to a polynomially high precision say $O(1/n^3)$.

Using the above simple principles we show that one can recover the first layer $X_1$ just from the correlation matrix $YY^\top $.
We will prove for the linear case that if $Y$ is a product of many $d$-sparse matrices then one can factorize $Y$.

\begin{theorem}
If $Y = X_1 X_2\dots X_s (1/\sqrt d)^s $ and each $X_i$ is a random $d$-sparse matrix, then there is an algorithm to compute $X_1$ from $Y$ with high probability when $n^{o(1)} \le d \le \tilde O(n^{1/6})$ and $s \le \tilde O(\sqrt n/d)$.
\end{theorem}

Observe that if $X$ is well-conditioned then by pre-multiplying $Y$ by $X_1^{-1}$ we get $X_2\dots X_n$ and can repeatedly invoke the above theorem  at successive levels. However bounds on extreme singular values are not known for sparse random matrices. For a random $\pm 1$ matrix $X$, it is known \cite{rudelson2009smallest,tao2010random} that with high probability of $1-\eps$ the smallest and largest singular values are at least $\eps / \sqrt n$ and at most $O(\sqrt{n})$. In a recent work, \cite{wood2012universality} extends the circular law on the distribution of the eigenvalues to that of sparse random matrices but it does not establish lower bounds on the smallest eigenvalue. Thus we have the following:

\begin{theorem}
Let $Y = X_1 X_2\dots X_s (1/\sqrt d)^s $ and each $X_i$ is a random $d$-sparse matrix, then w.h.p either one of the $X_i$
s have a low condition number more than $O(n^2)$ or
there is an algorithm to compute the factors $X_1,\dots, X_s$ (where the columns are correct upto negation) in polynomial time when $n^{o(1)} \le d \le \tilde O(n^{1/6})$ and $s \le \tilde O(\sqrt n/d)$.
\end{theorem}
Note that the network is constructed bottom up, but the examples in $Y$ are generated top to down. Just as has been pointed out in \cite{arora2} the network has a certain {\it reversibility} property: if the input $x$ produces an output $y$ by going down the network, then given an output vector $y$, one can reconstruct the hidden input vector $x$ by going in the reverse direction up the network. However there is a small modification as one goes up layer by layer -- the modification involves applying some iterative corrections by going back and forth along each layer (see appendix \ref{sec:reversibile}).

\section{Algorithm}

Our main observation is that one can compute $X_1 X_1^\top$ by looking at $YY^\top $ and rounding it to an integer. From $X_1 X_1^\top $ one can recover $X_1$. If $X_1$ has a bounded condition number, it can be inverted and one can solve for $X_2 X_3\dots X_s$ and continue like this iteratively to find the rest.

For ease of exposition we will use a different notation for the first matrix $X_1$ than the rest.
\begin{lemma}
Let $Y = XZ_1\dots Z_\ell (1/\sqrt d)^\ell$ where each of the matrices $X, Z_1,\dots,Z_\ell$ is a random $d$-sparse matrix. Then the non-diagonal entries of the correlation matrix $XX^\top $ are equal to $round(YY^\top )$ w.h.p. where the $round()$ function rounds a real number to the nearest integer.
\end{lemma}

Define $Z=Z_1\dots Z_\ell (1/\sqrt d)^{\ell}$. Note that $ZZ^\top $ is equal to the identity matrix in expectation. Now if the eigenvalues of $ZZ^\top $ were close to $1$, then $YY^\top  = XZZ^\top X^\top $ would be close to $XX^\top $. Unfortunately just the bounds in the eigenvalues
of $ZZ^\top $ are not sufficient to recover $XX^\top $. Further, dependencies are created in the columns of $Z$ from the several matrix multiplications. Despite these challenges we show that $XX^\top $ can be recovered from $YY^\top $ by a simple rounding.


\section{Distribution of entries in $YY^\top $}
Throughout this section, define $Y=XZ$ where $Z=Z_1\dots Z_\ell (1/\sqrt d)^{\ell}$ and $Z_1,\dots,Z_\ell$ are random $d$-sparse matrices. We will characterize the distribution of a random variable $R$ by its characteristic function $\Phi_R(t) = E[e^{tR}]$ \footnote{Characteristic function is usually defined a bit differently: $E[e^{itR}]$. Note that for two independent random variables $R_1, R_2$, $\Phi_{R_1+R_2}(t) = \Phi_{R_1}(t) \Phi_{R_2}(t)$ }. The joint characteristic function of two random variables $R_1,R_2$ can also be defined as $\Phi_{R_1,R_2}(s,t) = E[e^{sR_1+tR_2}]$. For two polynomials $P(t)$ and $Q(t)$, we will say that $P(t) \preceq Q(t)$ if each coefficient in $P(t)$ is less than or equal to the corresponding coefficient in $Q(t)$. We also define $H(P(t))$ as the truncation of the polynomial $P(t)$ up to degree at most 2.

First, to simplify the analysis we will study the properties of $YY^\top $ when each entry of $X$ is generated from the gaussian $\NN(0,1)$ distribution. Then, will extend our methods to the case when $X$ is a random $d$-sparse matrix.


\subsection{When $X$ is a gaussian random matrix}

For the gaussian case, we will prove that w.h.p. $YY^\top $ has small off diagonal entries and large diagonal entries that are well separated.
The following lemma is the main statement that we will prove in this section. 

\begin{lemma}\label{lem:uvgaussian}
Let $u,v$ denote two row vectors of $X$. If $u,v$ are independently drawn from $\NN(0,1)^n$ then
w.h.p. the following hold:
$\big|\frac{uZZ^\top v^\top }{n}\big| \le \tilde O(\ell / \sqrt n)$ and
$\big|\frac{uZZ^\top u^\top} {n}-1\big| \le \tilde O(\ell /\sqrt n).$
\end{lemma}

By induction, we prove that the distributions of row vectors $uZ$ and $vZ$ are close to $\NN(0,1)^n$.
First, we will bound the difference between each entry in the row vector $uZ$ and $\NN(0,1)$. Our bounds hold when we condition on $u,Z_1,\dots,Z_{\ell-1}$ with high probability.
Let $q_\ell = uZ_1\dots Z_\ell$. Note that conditioned on $u,Z_1,\dots,Z_{\ell -1}$, every coordinate of $q_\ell$ is independent and identically distributed. So we only need to study the distribution of $q_{\ell i}$. Let $D_{\ell}$ denote this distribution for given $u,Z_1,\dots,Z_{\ell-1}$ . We will bound the difference between the characteristic function of $q_{\ell i}$ (that is $\Phi_{D_\ell}(t)$) and characteristic function of the normal distribution.

\begin{lemma}
\label{lem:single-gaussian}
If $u$ is distributed as $\NN(0,1)^n$, then with high probability over the values of $u,Z_1,\dots ,Z_{\ell-1}$, $\Phi_{D_\ell} (t) \le e^{\frac{t^2}{2}+\frac{\ell t^2c^4\log^2 n}{\sqrt n}}$ for any $t$ where $|t|\leq \sqrt{d}/\log^2 n$ where $c$ is a constant. Further, w.h.p. the maximum value of $q_{\ell i}$ is at most $\sqrt{c\log n}$ for $\ell \leq \tilde{O}(\sqrt{n})$.
\end{lemma}
\begin{proof}
We use induction on $\ell$. At the base level, $q_0=u$; so $\Phi_{D_0}(t)=\Phi_{\NN(0,1)}(t)=e^{t^2/2}$. $q_\ell$ is obtained from $q_{\ell-1}$ in the following two steps: first $n$ random samples $Q_1,\dots,Q_n$ are drawn from the distribution $D_{\ell-1}$. We will first prove that for a constant $c$ (that is the same for all layers), $Q_i \leq \sqrt{c\log n}$ with high probability for $\ell \leq \frac{\sqrt{n}}{4c^4\log^2 n}$. Note that using inductive hypothesis we know that:
$$
E[e^{tQ_i}] \leq e^{\frac{t^2}{2} + \frac{\ell t^2c^4\log^4n}{\sqrt{n}}} \leq e^{3t^2/4}
$$
So using Markov inequality we have:
$$
P( e^{tQ_i} \geq e^{t\sqrt{c\log n}}) \leq \frac{e^{3t^2/4}}{ e^{t \sqrt{c\log n}}}
$$
For $t=\frac{2}{3}\sqrt{c\log n}$ the above inequality will be bounded by $n^{-c/3}$ that is polynomially small in $n$.
Then $d$ random numbers are drawn from this set. Next, we take a linear combination, each one multiplied by a random sign and finally divide the result by $\sqrt{d}$. Thus the characteristic function $\Phi_{D_\ell }$ can also be obtained from $\Phi_{D_{\ell-1}}$ in two steps. Let $\tilde{Q}$ denote the random variable $\alpha Q_i$ where $\alpha$ is a random sign and $i$ is a random index from 1 to $n$. Let $P_\ell$ denote the characteristic function for $\tilde{Q}$ conditioned on given values $Q_1, \dots,  Q_n$ which correspond to the vector $q_\ell$. Then $D_{\ell+1}$ is obtained by adding $d$ such $\tilde{Q}$s and dividing by $\sqrt{d}$. So $\Phi_{D_\ell}(t)=(P_\ell(t/\sqrt{d}))^d$ and note that $P_\ell(t)=\frac{1}{n}\sum_i (e^{tQ_i}+e^{-tQ_i})/2$. Since $Q_i$s are independent, this is the average of $n$ identically distributed independent random values each with mean $E[(e^{tQ_i}+e^{-tQ_i})/2)]=\Phi_{D_{\ell-1}}(t)$. We will bound the difference from the mean with high probability.

$$
(e^{tQ_i}+e^{-tQ_i})/2=1+t^2Q_i^2+t^4Q_i^4/4!+t^6Q_i^6/6!+\dots
$$
also
$$
E[(e^{tQ_i}+e^{-tQ_i})/2]=1+t^2E[Q_i^2]+t^4E[Q_i^4]/4!+t^6E[Q_i^6]/6!+\dots.
$$

Note that the odd powers of $t$ will not be present because the probability density function is even. So the difference will be:
$$
|(e^{tQ_i}+e^{-tQ_i})/2 - E[(e^{tQ_i}+e^{-tQ_i})/2]| = t^2\frac{Q_i^2-E[Q_i^2]}{2!}+t^4\frac{Q_i^4-E[Q_i^4]}{4!}+\dots
$$
We also know that $Q_i \leq \sqrt{c\log n}$ with high probability. So $Q_i^2-E[Q_i^2]$ is a bounded random variable with absolute value at most $c\log n$. So the average of $n$ such random variables is at most $c^4\log^2 n/\sqrt{n}$ with high probability. Similarly, with high probability the average value of $Q_i^4-E[Q_i^4]$ is at most $c^8\log ^4 n/\sqrt{n})$. Thus with high probability the difference will be at most:
$$
|(e^{tQ_i}+e^{-tQ_i})/2 - E[(e^{tQ_i}+e^{-tQ_i})/2]| \leq t^2 \frac{c^4\log^2 n}{2!\sqrt{n}} + t^4 \frac{c^8\log^4 n}{4!\sqrt{n}} + \dots
$$
that is at most $t^2 \frac{c^4\log^2 n}{\sqrt{n}}$ if $t\leq \frac{1}{c^2\log n}$. Now $\Phi_{D_\ell}(t) = (P_\ell(t/\sqrt{d}))^d$. For $t \leq \frac{\sqrt{d}}{c^2\log n}$ we get that this is at most $\Phi_{D_{\ell-1}}(t)+ t^2 \frac{c^4\log^2 n}{d\sqrt{n}}$ which is by induction bounded by:
\begin{eqnarray*}
\bigg(e^{\frac{t^2}{2d} + \frac{\ell t^2c^4\log^2 n}{d\sqrt{n}}} +t^2 \frac{c^4\log^2 n}{\sqrt{n}}\bigg)^d &\leq \bigg(e^{\frac{t^2}{2d} + \frac{\ell t^2c^4\log^2 n}{d\sqrt{n}}}\bigg)^d\bigg(1 +t^2 \frac{c^4\log^2 n}{\sqrt{n}}\bigg)^d\\
&\leq e^{\frac{t^2}{2} + \frac{\ell t^2c^4\log^2 n}{d\sqrt{n}}} e^{t^2 \frac{c^4\log^2 n}{d\sqrt{n}}}\\
&\leq e^{\frac{t^2}{2} + \frac{(\ell+1)l t^2c^4\log^2 n}{\sqrt{n}}}\\
\end{eqnarray*}
\end{proof}

Next we study the joint distribution of $q_\ell = uZ_1\dots Z_\ell$ and $w_\ell = vZ_1\dots Z_\ell$. Define $\Gamma_\ell$, in the same way as $D_\ell$, to be the distribution of $w_{\ell i}$ conditioned on the values of $v$, $Z_1$, $\dots$, $Z_{\ell -1}$. Again, we first study this when $u,v$ are normally distributed. We will look at the joint characteristic function of two random variables $q_{\ell i},w_{\ell i}$ denoted by $\Phi_{D_\ell,\Gamma_\ell}(s,t)$. A similar analysis gives the following lemma where we bound the coefficients of this characteristic function up to degree 2.  The proof is based on the similar techniques as lemma \ref{lem:single-gaussian} (see appendix \ref{sec:joint}).


\begin{lemma}
\label{lem:termwise-gaussian}
If $u$ and $v$ are independently distributed as $\NN(0,1)^n$, then with high probability over the values of $u,v,Z_1,\dots ,Z_{\ell-1}$, 
$$1+\frac{s^2+t^2}{2} - \frac{\ell \log^2 n}{2\sqrt{n}}(s+t)^2\preceq H(\Phi_{D_\ell,\Gamma_\ell}(s,t)) \preceq 1+\frac{s^2+t^2}{2} + \frac{\ell \log^2 n}{2\sqrt{n}}(s+t)^2$$
\end{lemma}

The following two lemmas show that the diagonal and non-diagonal entries of $YY^\top $ are far apart.

\begin{lemma}
If $u$ and $v$ are independently distributed as $\NN(0,1)^n$,  then
w.h.p., $|uZZ^\top v^\top| \le c^4 (\ell+1) \log^2 n\sqrt{n}$.
\end{lemma}
\begin{proof}
Using lemma \ref{lem:single-gaussian}, we know that if $q_{\ell}=uZ$ and $w_{\ell}=vZ$, then conditioned on $u,v,Z_1,\dots,Z_{\ell-1}$, for all $i$, $q_{\ell i}w_{\ell i}$ are independent variables and bounded by $c\log n$ with high probability. Moreover, $E[q_{\ell i }w_{\ell i}]$ is nothing but the coefficient of $st$ in $\Phi_{D_\ell,\Gamma_{\ell}}(s,t)$. By lemma \ref{lem:termwise-gaussian}, we know that this coefficient has the following bound:
$$
|E[q_{\ell i}w_{\ell i}]| \leq \frac{c^4 \ell \log^2 n}{\sqrt{n}}
$$
Now, using Hoeffding's bound, we have that:
$$
P( | \sum_i ( q_{\ell i}w_{\ell i} - E[q_{\ell i}w_{\ell i}])| > t )\leq 2\exp(-t^2/nc^2\log^2 n)
$$
So with high probability:
$$
uZZ^\top v^\top  = \sum_i q_{\ell i}w_{\ell i} \leq c^4 \ell \log^2 n\sqrt{n} + c^4\log^2 n\sqrt{n} = c^4 (\ell+1) \log^2 n\sqrt{n}
$$
\end{proof}
			
\begin{lemma}
If $u$ is a random vector distributed as $\NN(0,1)^n$ then w.h.p., $uZZ^\top u^\top \in n\pm c^4 (\ell+1) \log^2 n\sqrt{n}$\footnote{For simplicity in notation, we denote the inequalities of the form $\alpha + \beta \leq t \leq \alpha + \beta$ by $t \in \alpha \pm \beta$.}.
\end{lemma}
\begin{proof}
Let $q_{\ell}=uZ$. By lemma \ref{lem:single-gaussian} and \ref{lem:termwise-gaussian} we know that conditioned on the value of $u,Z_1,\dots,Z_\ell$, for all $i$, $q_{\ell i}^2$ are independent variables bounded by $c\log n$ and w.h.p,  $|E[q_{\ell i}^2]-1| \leq c^4 \ell \log^2 n/\sqrt{n}$.
By applying Hoeffding's bound, with high probability we have:
$uZZ^\top u = \sum_i q_{\ell i}^2 \in n \pm c^4 (\ell+1) \log^2 n\sqrt{n}$
\end{proof}

Next we extend this to the case when $u$ and $v$ are random $d$-sparse vectors.


\subsection{When $u$ and $v$ are random $d$-sparse vectors}

We will now study the distribution of $q_\ell =\sqrt {n/d} uZ_1\dots Z_\ell (1/\sqrt d)^\ell$ for $\ell \le \log_d n - 1$. We bound $\Phi_{D_\ell }$ where $D_\ell $ is the distribution of $q_{\ell i}$ conditioned on $u,Z_1,\dots ,Z_{\ell -1}$. For ease of exposition, we will assume that $\log_d n$ is an integer.

\begin{lemma}
For $\ell \le \log_d n - 1$, with high probability the following hold:
\begin{itemize}
\item The number of non-zero entries in $q_\ell$ is at most $O(d ^{\ell +1}/n)$.
\item The maximum value of any entry in $q_\ell$ is at most $\frac{\sqrt{n}}{d}(c \log n/d)^\ell$.
\item $\Phi_{D_\ell }(t) \le 1 + \sum_{j\ge 1}\frac{ O( [t (c\log n) ^{\ell }]^{2j})}{(2j)!}$.
\end{itemize}
where $c$ is a constant independent of $\ell$.
\end{lemma}

\begin{proof}

First for ease of exposition define $q_\ell = uZ_1\dots Z_\ell$ without the scaling factors of $\sqrt {n/d}$ for $u$ and $1/\sqrt d$ for each $Z_i$. So each entry in $q_{\ell i}$ is obtained by signed linear combination of $d$ random entries in $q_{\ell -1}$. Now by induction with high probability the number of non-zero entries in $q_{\ell i}$ is at most $r_\ell = (d + O(\sqrt {d \log n}) )^{\ell +1}$. For convenience we define $q_0 = u$. Then this is true at $\ell =0$ since $u$ is $d$-sparse. And since the next layer is formed by adding up $d$ random entries of $y_{\ell -1}$ w.h.p., the number of non zero entries in $q_\ell$ is at most $dr_\ell + \sqrt{d r_\ell \log n} \le (d + O(\sqrt {d \log n})) ^{\ell +1}$. For $\ell \le\log_d n$ and $d >\log^2 n$ this is at most $O(d ^{\ell +1}/n)$.

Next by induction the maximum value of $q_{\ell i}$ is at most $(c \log n)^\ell$. This is because each entry is expected to touch $d r_{\ell -1} /n$ non zero entries which is at most $1$ for $\ell \le\log_d n -1$ and so with high probability it will touch at most $c \log n$ entries in $q_{\ell -1}$. So the $jth$ moment of $q_{\ell i}$ is at most $O(d ^{\ell +1}/n) (c \log n)^{jl}$.

These can be used to get simple bounds on the moments of $q_{\ell i}$ conditioned on $u,Z_1,\dots ,Z_{\ell -1}$.
So $\Phi_{D_\ell }(t) \le 1 + O(d ^{\ell +1}/n) ( \sum_{j\ge 1} (t (c\log n) ^{\ell })^{2j}/((2j)!)$.
Switching back to the right scaling factors completes the proof.
\end{proof}

Let $M = O( \log n)^{\log_d n}$. Then at $\ell = \log_d n-1$, $\Phi_{D_\ell }(t) \le e^{t^2 M^2}$. Further for $d = n^{\omega(\sqrt{\log\log n/\log n})}$, $M = d^{o(1)}$.

Next we will bound the characteristic function of the joint distribution $(D_\ell , \Gamma_\ell )$ up to degree 2 where $u,v$ are disjoint and
$w_\ell = \sqrt{n/d} vZ_1\dots Z_\ell (1/\sqrt d)^\ell$. Again we will condition on $u,v,Z_1,\dots ,Z_{\ell -1}$. See the appendix \ref{sec:sparse} for the proof of the following lemma.

\begin{lemma}\label{lem:sparselogdn}
At $\ell=\log_d n -1$, $1+(s^2+t^2)/2 - \delta_1(s^2+t^2) - \delta_2 st \preceq H(\Phi_{D_\ell,\Gamma_\ell}(s,t) \preceq 1+(s^2+t^2)/2 + \delta_1(s^2+t^2) + \delta_2 st$ where $\delta_1 = \tilde O(1/\sqrt d)$ and $\delta_2 = M/d^{2}$.
\end{lemma}

Now we use the lemma \ref{lem:sparselogdn} to prove similar statements to lemmas \ref{lem:single-gaussian} and \ref{lem:termwise-gaussian} for higher layers in the sparse case.
\begin{lemma}
\label{lem:single-sparse}
If $u$ is a random $d$-sparse vector, then with high probability over the values of $u,Z_1,\dots ,Z_{\ell-1}$, $\Phi_{D_{\ell}} (t) \le e^{M^2\frac{t^2}{2}+\frac{\ell M^2t^2c^4\log^2 n}{\sqrt n}}$ for any $t$ where $|t|\leq \sqrt{d}/(M\log^2 n)$ where $c$ is a constant . Further, the maximum value of $q_{\ell i}$ is at most $M\sqrt{c\log n}$ for $\ell \leq \tilde{O}(\sqrt{n})$.
\end{lemma}
The proof steps are very similar to inductive proof of lemma \ref{lem:single-gaussian} except that the maximum value of $q_{\ell i}$ is bounded by $M\sqrt{c\log n}$ instead of $\sqrt{c\log n}$. We use lemma \ref{lem:sparselogdn} to prove the statement at base level $\ell=\log_d n -1$.
Next, we prove the adaptation of lemma \ref{lem:termwise-gaussian} to the sparse $u,v$ case.

\begin{lemma}\label{lem:sparsedl}
If $u,v$ are independent random $d$-sparse vectors, then with high probability over the values of $u,v,Z_1,\dots ,Z_{\ell-1}$, $1+\frac{s^2+t^2}{2} - \delta_1\frac{s^2+t^2}{2} - \delta_2 st - \frac{\ell M^2\log^2 n}{\sqrt{n}}(s+t)^2 \preceq H(\Phi_{D_\ell}(s,t)) \preceq 1+\frac{s^2+t^2}{2} + \delta_1\frac{s^2+t^2}{2} + \delta_2 st+ \frac{\ell M^2\log^2 n}{\sqrt{n}}(s+t)^2$.
\end{lemma}
This proof is also very similar to the inductive proof of lemma \ref{lem:termwise-gaussian} and again we use \ref{lem:sparselogdn} at the base level $\ell =\log_d n -1$.


\subsection{Recovering $XX^\top $ from $YY^\top $}
We have established that $H(\Phi_{D_\ell,\Gamma_\ell}(s,t))$ is termwise within $1+\frac{s^2+t^2}{2} \pm \epsilon_1(s^2+t^2) \pm \epsilon_2 st$ where $\eps_1 = \tilde O(1/\sqrt d + \ell M^2/\sqrt n)$ and $\eps_2 = \tilde O(M/d^2 + \ell M^2/\sqrt n)$. Now we will prove that:
\begin{lemma}
\label{lem:uv}
If $u$ and $v$ are random $d$-sparse vectors, then w.h.p. $(uZZ^\top v^\top )/n \in uv^\top  \pm \tilde{O}(\epsilon_1 +d\epsilon_2 + dM/\sqrt{n})$.
\end{lemma}

For $\ell \le \tilde O(\sqrt n/(dM^2)) = \sqrt n/d^{1+o(1)}$ the difference from $u^\top v$ is o(1). So rounding it to the nearest integer will give exactly $uv^\top $.

\begin{lemma}
If $u,v$ are disjoint (that is do not share a non-zero column) then
w.h.p. $|uZZ^\top v| \le O(d\epsilon_2 + dM^2\log^2 n/\sqrt{n})$
\end{lemma}
\begin{proof}
It is sufficient to prove this for $\ell\geq \log n$ because for $\ell < \log n$ we can always multiply $uZ_1\dots Z_\ell$ from the right side by enough number of artificially random matrices. Now by lemma \ref{lem:sparsedl} the coefficient of $t^2$ in $\Phi_{D_\ell}(t)$ be in the range $\frac{1}{2} \pm \epsilon_1$. Assume $q_1,\dots,q_n$ are $n$ samples generated from this distribution. We also know that $q_i \leq M\sqrt{c\log n}$ with high probability. The same bound holds for $W_i$. So $|q_iw_i|$ is bounded by $M^2c\log n$. $E[q_iw_i]$ is the coefficient of $st$ in $\Phi_{D_\ell,\Gamma_\ell}(s,t)$ which is $\epsilon_2$. So using Hoeffding's inequality, we have that with high probability:
$$
|\sqrt{\frac{n}{d}}uZZ^\top v^\top \sqrt{\frac{n}{d}}| = |\sum_ i q_iw_i| \leq n\epsilon_2 + c^3M^2\sqrt{n}\log^2 n
$$
\end{proof}

\begin{lemma}\label{lem:sparseuu}
If $u$ is a sparse vector then
w.h.p. $|uZZ^\top u'|/d \in 1 \pm O(\epsilon_1+M^2\log^2 n/\sqrt n)$
\end{lemma}
\begin{proof}
Let $q=\sqrt{\frac{n}{d}}uZ$. We know that entries $q_i$ in the vector $q$ are independent conditioned on the value of $u,Z_1,\dots,Z_{\ell-1}$. Moreover, $1-\epsilon_1 \leq E[q_i^2] \leq 1+\epsilon_1$ and the absolute value of $q_i$ is bounded by $M\sqrt{c\log n}$. Now using, Hoeffding's inequality with high probability, we following bound holds:
$$
\sqrt{\frac{n}{d}}uZZ^\top u^\top \sqrt{\frac{n}{d}}=\sum_i q^2_i \in n\pm n\epsilon_1\pm c^2M^2\sqrt{n}\log^2 n
$$
\end{proof}

We will now give the proof of lemma \ref{lem:uv}.

\begin{proof}(of lemma \ref{lem:uv})
Without loss of generality, assume that $u$ and $v$ are both non-zero at the first $k$ entries and not simultaneously non-zero in the remaining entries. Now let $\tilde{u}$ be a vector such that for any $i\leq k$, $\tilde{u}_i=u_i$ and $\tilde{u}_i=0$ otherwise and define $\tilde{v}$ in the same way. We have that:
\begin{eqnarray*}
uZZ^\top v^\top  &=& \sum_{i\leq k}u_iv_ie_iZZ^\top e_i^\top  + \sum_{i\neq j,i,j\leq k}u_iv_je_iZZ^\top e^\top _j \\
&+& \tilde{u}ZZ^\top (v-\tilde{v})^\top  + (u-\tilde{u})ZZ^\top \tilde{v}^\top  + \tilde{u}ZZ^\top \tilde{v}^\top \\
\end{eqnarray*}
Note that $e_iZZ^\top e_i^\top =xZ_2\dots Z_\ell Z^\top _\ell \dots Z_2^\top x^\top /d^{\ell/2}$ where $x$ is $i$th row of $Z_1$. So the bound from lemma \ref{lem:sparseuu} applies:
$$
\sum_{i\leq k}u_iv_ie_iZZ^\top e^\top _i = uv^\top  \pm O(\epsilon_1 + M^2\log^2 n/\sqrt{n})k
$$
In all remaining $O(k^2)$ terms, vectors are disjoint and with high probability, the sum of their absolute values  is bounded by
$$
d(k^2+3)(\epsilon_2 +M^2\log^2 n /\sqrt{n})
$$
Since with high probability, $k=O(\log n)$ the bound will be at most:
$$
d(\epsilon_2\log^2 n +M^2\log^4 n /\sqrt{n})
$$
The lemma then follows from adding the bound on the terms:
$$
uZZ^\top v^\top  = uv^\top +d(\epsilon_2\log^2 n +M^2\log^4 n /\sqrt{n}) + O(\epsilon_1 + M^2\log^2 n/\sqrt{n})k
$$
\end{proof}



\section{Obtaining $X$ from $XX^\top $}
In this section, we show the following
\begin{theorem}\label{thm:getX}
There is an algorithm that correctly recovers $X$ from $XX^\top $ with high probability.
\end{theorem}

The following algorithm can be used to recover $X$ from $XX^\top $. Note that $XX^\top $ gives us the correlations among $n$ rows of $X$ that were obtained by rounding the correlations among the rows of $Y$. We will show how to reconstruct the layer of weighted edges (corresponding to non-zero entries of $X$). The edges need to be recovered between the hidden nodes (that correspond to $n$ columns of $X$) and the outputs (corresponding to the $n$ rows of $X$). Assume we have identified a correlated pair of entries $X_i,X_j$ whose supports intersect in exactly one column say $g$. The following algorithm recovers the column $g$ denoted by $(X^\top )_g$ (upto negation) by connecting the edges to the corresponding hidden node.

{\bf Initial Step:} a)Create a new hidden node $p$ and connect it to output nodes $i$ and $j$ that are correlated.
b)Connect the new hidden node to all output nodes $k$ that are correlated to to both $i$ and $j$.

If nodes $i$ and $j$ have exactly one common neighbor to the layer above (that is they share one non-zero column) then the above algorithm constructs one hidden node correctly except that it may miss $o(1)$ fraction of the nodes. Let $S$ denote the set of nodes under the new hidden node $p$. Let $supp(v)$ denote the support of vector $v$.

\begin{claim}\label{lem:oneintersect}
If $X_i$ and $X_j$ share exactly one non-zero column $g$, then $|S - \supp((X^\top )_g)| \le o(d)$
\end{claim}

The following simple pruning can be used to fix the erroneous nodes.

{\bf Prune Step:} Drop output nodes $k$ that are not correlated to most nodes in $S$ (more than $1-o(1)$ fraction) and add output nodes $k$ that are correlated to
most nodes in $S$.

\begin{claim}\label{lem:oneintersect}
If $X_i$ and $X_j$ share exactly one non-zero column $g$ then after the prune step the set of nodes that remain is identical to $supp((X^\top )_g)$. Otherwise the set is empty.
\end{claim}

Thus we can get one hidden node for each column $g$ by considering all pairs $X_i, X_j$.

{\bf Obtaining edge weights:} The weights on the edges can be obtained as follows: First set the sign of the weight of an edge from a node $k$ to the sign of its correlation to node $i$. Again, this will give the correct sign $w_k$ for most ($1-o(1)$ fraction) of edges. Next, flip $w_k$ if the correlation between $w_k X_k$ and $w_{k'}^\top X_{k'}^\top$ is negative for most $k' $ in $S$. Now all the wrong signs get corrected. The magnitude of the weight $w_k$ is set to the majority value of $w_k w_{k'} X_k {X _{k'}}^\top$ over $k'  \in S$.

The proof of the correctness of the algorithm is in appendix \ref{sec:recovery}.

\section{Conclusions}
We studied the problem of Sparse Matrix Factorization and  explored
its relationship to learning deep networks. For the linear case and
for sparse random matrices we showed a simple natural algorithm that
is able to  reconstruct the linear deep network (that corresponds to
factorizing the matrix) by simply finding correlated nodes in the
lower layer. This works as long as the sparsity and the depth are
under $O(n^{1/6})$.

In terms of future directions, it would be interesting to find
natural algorithms that can reconstruct non-linear networks (that
corresponds to non-linear factorization) for large depths -- this is
already known for depths up to $O(\log_d n)$. Another interesting
direction is to find algorithms that work  other types of
distributions besides the sparse random matrices or for lower levels
of sparsity.


\bibliographystyle{plain}
\bibliography{smf}

\appendix


\section{Circuits, Deep Networks and Compression}\label{sec:kolcircuit}

A neural network with $s$ layers of nodes can be represented using matrices $X_1,\dots ,X_s$. Here $X_s$ is the matrix of inputs (each column is a separate input) and $Y = \sigma( X_1. \sigma( X_2 . \sigma( \dots \dots X_s)))$ is the matrix of outputs.
On an input $z$ from the upper layer the $\ell th$ layer of edges produces an output $y = \sigma(X_\ell z)$ and gives it to the lower layer. Instead of thinking of $X_s$ as a matrix of inputs one can also think of it as a layer of edges -- then to produce the $j$th column in $Y$, instead of feeding in the $j$th column of $X_s$ as input at the top, one just feeds the vector $e_j$ (that is $1$ at the $jth$ coordinate and $0$ elsewhere). This produces a network where each output vector is produced by turning on exactly one input node.

More generally one can ask the ideal circuit using and/or/not gates that represents the images in the matrix $Y$ where each image corresponds to turning on one input at the top layer of the circuit. This is just a circuit version of the Kolmogorov complexity (Kolmogorov complexity of a binary string is the smallest turing machine that outputs that string). Given a circuit with $m$ inputs and $n$ outputs we will say that it produces the matrix $Y$ if $i$th column in $Y$ is output by turning on the $ith$ input in the circuit (and setting the other inputs to $0$).

\begin{Problem}[Circuit-Kolmogorov Complexity of a Matrix]
Given an $n \times m$ matrix $Y$ with $0,1$ entries, find the circuit (using and/or/not gates) with the smallest number of edges that produces matrix $Y$. One may also restrict to layered circuits where edges are present only between consecutive layers.
\end{Problem}

By using either the sigmoid or the step function or the sign function at each node, one can express any circuit. In fact this is possible even when the weights are restricted to only $0, \pm 1$. Any (layered) circuit can be converted to a neural network and vice versa with at most $O(\log n)$ blow up in size. This is because the thresholded sum of $k$ bits can be computed using a layered circuit of size at most $\tilde O(k)$.

The smallest circuit that produces a matrix can be viewed as a compressed representation of the matrix. The number of bits required to encode the circuit is not exactly the the number of edges but close; one will need $O( \log n)$ blow up to express the node id's that the edge connects. Thus in this sense one can think of the number of edges in the circuit as capturing the circuit version of the Kolmogorov complexity of the matrix.

Since the sign function can be used to emulate and/or/not gates
\begin{observation}
When restricted to layered circuits, the Circuit-Kolmogorov-complexity of a matrix $Y$ is equivalent to the optimal Non-linear sparse matrix factorization upto $\tilde O(1)$ multiplicative factors.
\end{observation}

Pseudorandom number generators can provably produce a matrix of pseudo random bits from a smaller input matrix of random bits assuming the existence of one-way functions. Thus being able to compress the output matrix would correspond to inverting one way functions which is as hard a integer factoring.


\section{Reversibility of Deep Networks}\label{sec:reversibile}

\cite{arora2} argued that for a random matrix of edge weights, each layer of the neural network is reversible under certain conditions which matches the underlying philosophy of RBMs.
A similar argument is possible for the linear case if the weight matrix $X$ is invertible.
Note that if an output vector $y$ has been produced from input $z$ then $y = Xz$.

A natural method to try to recover $z$ from $y$ is to go back along edges in the reverse direction giving $X^\top y = X^\top Xz$.
Now if $X$ is a random $d$-sparse matrix then if appropriately scaled by $1/\sqrt d$, $X^\top X$ is equal to the identity matrix $I$ in expectation. Thus the expected error in computing $z$ in this way is $0$ over the randomness of $X$.

It is also possible to compute $Z$ exactly by iteratively correcting it as long as $X$ is not a singular matrix; this corresponds to the simple standard iterative algorithm for solving a linear system.
We iteratively compute values $z_1,\dots ,z_k$ at the top layer that converge to the true value of $z$.
Initialize $z_1 = X^\top y$ by just computing the network backwards. Then correct $z_i$ by propagating a fraction of the error $y - X z_i$ backwards. That is $z_{i+1} = z_i + X^\top  \gamma (y - X z_i)$
Over $k$ iterations the error $y - X z_k = (I - \gamma XX^\top ) (y - X z_{k-1}) = (I - \gamma XX^\top )^k (y - X z_1)$

Now if $X$ is not singular then by setting $\gamma$ to be smaller than the largest eigenvalue of $XX^\top $ we get a matrix with eigenvalue strictly between $0$ and $1$. So if the number of iterations $k$ is much more than the condition number, the error converges to $0$.

For random dense matrices it is known that their condition number if polynomially bounded with high probability. However this is not been proven yet for sparse random matrices.


\section{Proof of lemma~\ref{lem:termwise-gaussian}}\label{sec:joint}
We first need to find an upper bound on the the joint characteristic function of $q_{\ell i},w_{\ell i}$:
\begin{lemma}
\label{lem:joint-gaussian}
If $u,v$ are independently distributed as $\NN(0,1)^n$, then with high probability over the values of $u,v,Z_1,\dots ,Z_{\ell-1}$, $\Phi_{q_{\ell i}, w_{\ell i}}(s,t) \le e^{\frac{s^2+t^2}{2}+\frac{\ell (s+t)^2c^4\log^2 n}{2\sqrt n}}$.
\end{lemma}
\begin{proof}
We prove this lemma by induction. At the base level $\Phi_{D_\ell,\Gamma_\ell}(s,t)=e^{(s^2+t^2)/2}$. So the lemma statement clearly holds at the base level. Assume that this is true up to $\ell-1$. By lemma \ref{lem:single-gaussian}, with high probability, all entries of $Q_{\ell-1}$ and $W_{\ell-1}$ are bounded by $\sqrt{c\log n}$. So similarly, we have that the difference is at most:
$$
s^2c^4\log^2 n/(\sqrt{n}2!) + t^2c^4\log^2 n/(\sqrt{n}2!) + stc^4\log^2 n/(\sqrt{n}) + \dots
$$
that is bounded by $e^{\frac{(s+t)^2c^4\log^2 n}{2\sqrt n}}$. Now, using the inductive hypothesis as before, we bound the final characteristic function.
\end{proof}

We will now bound each coefficient in $\Phi_{D_\ell,\Gamma_\ell}$ for terms of degree 2.
\begin{proof}(of lemma~\ref{lem:termwise-gaussian})
The assumption clearly holds at the base level. Let $\epsilon_{s,t}$ denote the term $ (s+t)^2\frac{c^4\log^2 n}{2\sqrt{n}}$. According to the proof in lemma \ref{lem:joint-gaussian}, at layer $\ell$, we have that:
$$
H(\Phi_{D_\ell}(s,t)) -\epsilon_{s,t} \preceq H(P_{\ell+1}(s,t)) \preceq H(\Phi_{D_\ell,\Gamma_\ell}(s,t)) +\epsilon_{s,t}
$$
where $\epsilon_{s,t} = \frac{\log^2 n}{\sqrt{n}}(s+t)^2$. By induction:
$$
1+(s^2+t^2)/2 - \ell\epsilon_{s,t}\preceq H(\Phi_{D_\ell,\Gamma_\ell}(s,t)) \preceq1+(s^2+t^2)/2 + \ell\epsilon_{s,t}
$$
So we have:
$$
1+(s^2+t^2)/2 - (\ell+1)\epsilon_{s,t}\preceq H(P_{\ell+1}(s,t)) \preceq 1+(s^2+t^2)/2 +(\ell+1)\epsilon_{s,t}
$$
We know $\Phi_{D_{\ell+1},\Gamma_{\ell+1}}(s,t) = (P_{\ell+1}(s/\sqrt{d},t/\sqrt{d}))^d$. Therefore:
$$
H\bigg[\bigg(1+\frac{s^2+t^2}{2}- (\ell+1)\epsilon_{s,t}\bigg)^d\bigg]\preceq H(\Phi_{D_{\ell+1}}(s,t)) \preceq H\bigg[\bigg(1+\frac{s^2+t^2}{2} + (\ell+1)\epsilon_{s,t}\bigg)^d\bigg]
$$
the truncation of up to degree 2 gives us:
$$
1+(s^2+t^2)/2 - (\ell+1)\epsilon_{s,t}\preceq H(\Phi_{D_{\ell+1},\Gamma_{\ell+1}}(s,t)) \preceq 1+(s^2+t^2)/2+(\ell+1)\epsilon_{s,t}
$$
\end{proof}


\section{Proof of lemma \ref{lem:sparselogdn}}\label{sec:sparse}
Lemma \ref{lem:sparselogdn} is the direct result of the following two lemmas:

\begin{lemma}\label{lem:phisparse1}
For $\ell \le \log_d n -1$, $1 -O((c^2 \log^2 n)^\ell) (s^2 +t^2) -\frac{O((4c^2 \log^2 n)^\ell)}{d^2}st \preceq H(\Phi_{D_\ell , \Gamma_\ell }(s,t))\preceq 1 + O((c^2 \log^2 n)^\ell) (s^2 +t^2) + \frac{O((4c^2 \log^2 n)^\ell)}{d^2}st$.

(Note that this means for  $\ell = \log_d n -1$ w.h.p.
$1 -O(M) (s^2 +t^2) -\frac{O(M)}{d^2}st \preceq H(\Phi_{D_\ell , \Gamma_\ell }(s,t))\preceq 1 + O(M) (s^2 +t^2) + \frac{O(M)}{d^2}st$)
\end{lemma}

\begin{proof}
Again for ease of exposition define $q_\ell = uZ_1\dots Z_\ell$ without the scaling factors of $\sqrt {n/d}$ for $u$ and $1/\sqrt d$ for each $Z_i$.
Now at $\ell =1$ since $u,v$ are disjoint, $q_1 = uZ_1$ and $w_1 = v Z_1$ are independent conditioned on given $u,v$ being disjoint since they do not share any entries from the matrix $Z_1$. So
$\Phi_{D_1,B_1}(s,t) = \Phi_{D_1}(s) \Phi_{B_1}(t)$. Note that the coefficient of $st$ is $0$.

We will bound the coefficient of $st$ (correlation between $q$ and $w$) in $\Phi_{D_\ell ,\Gamma_\ell }(s,t)$. We will prove by induction that the coefficient of $st$ is at most $d ^{\ell -1}/n (O(\log n)^\ell)$. We have already proven this for $\ell =1$.
Now at $\ell =2$, note that $P_\ell (s,t) = \frac{1}{2n} \sum_{i=1}^n (e^{sQ_i + tW_i} + e^{-sQ_i - tW_i})/2$ where each $(Q_i,W_i$) is sampled independently from the distribution $(D_{\ell -1},\Gamma_{\ell -1})$.

We will use $F_\ell $ to denote the set of coordinates in $q_\ell$ that have been influenced by(that is, have a path via the edges represented by $Z_i$'s to) some non-zero entry in $u$. More precisely, let $F_0$ be the set of non-zero coordinates in $u$. And recursively $F_\ell $ is the set of coordinate in $q_\ell$ which has used some coordinate in $F_{\ell -1}$ among the $d$ samples it makes. Similarly we define $G_\ell $ to be the similar set of coordinates in $w_\ell $. Let $S_\ell = F_\ell \cap G_\ell $.
We will prove by induction that w.h.p. $|S_\ell | \le (4d) ^{\ell -1} c \log n$ for $\ell \le \log_d n -2$. At $\ell =0$, $|S_0| = 0$ since $u$ and $v$ are disjoint in their non-zero positions. So this clearly holds. In going from $\ell -1$ to $\ell $, $S_\ell $ consists of those coordinates that touch $S_{\ell -1}$ or those that touch both $P_{\ell -1} \ S_{\ell -1}$ and $Q_{\ell -1} \ S_{\ell -1}$. Number of coordinates that touch $S_{\ell -1}$ is expected to be $d(4d) ^{\ell -1} c \log n$ and with high probability no more than $(1/2)(4d) ^{\ell -1} c \log n$. Since sizes of $P_{\ell -1} \ S_{\ell -1}$, $Q_{\ell -1} \ S_{\ell -1}$ are at most $O(d ^{\ell }/n)$ and disjoint, the expected number of coordinates that touch both is at most $ n d O(d^\ell /n) d O(d^\ell /n)$ and with high probability it is at most $O(d ^{\ell +3} \log n / n) d ^{\ell -1}$. For $\ell \le \log_d n - 3$ this is at most $c \log n d ^{\ell -2}$. Since the maximum value of entries in $q_\ell$ and $w_\ell $ are at most $(c \log n)^\ell$, the coefficient of $st$ in $P_\ell (s,t)$ is at most $(4d) ^{\ell -1} c \log n (c \log n)^{2l}$.
Thus we have shown that for $\ell = \log_d n - 3$
$H(P_{\ell +1}(s,t)) \preceq 1 + O( (d ^{\ell +1} (c^2 \log^2 n))^\ell /n (s^2 +t^2) + st (c^2 \log^2 n)^\ell (4d) ^{\ell -1} c \log n/n
= 1 + O(M^2/d^2) (s^2 +t^2) + O(M^2/d^4) st$

Now $\Phi_{D_{\ell +1}}(s,t) = (P_{\ell +1}(s,t))^d$ since we are not scaling by $1/\sqrt d$.
So $H(\Phi_{D_{\ell +1}}(s,t)) = H((H(P_{\ell +1}(s,t)))^d)$ since $H(P_{\ell +1})$ has zero coefficients for $s$ and $t$ terms.
Now $(H(P_{\ell +1}(s,t)))^d \preceq H((1 + O(M^2/d^2) (s^2 +t^2) + O(M^2/d^4) st)^d)$
Thus we get $\ell = \log_d n -2$, $H(\Phi_{D_{\ell }}(s,t)) \le 1 + O(M^2/d) (s^2 +t^2) + O(M^2/d^3) st $.

Finally in going to $\ell = \log_d n$, we note that $P_\ell $ is obtained by sampling $n$ pairs $(Q_i,W_i)$ from the distribution $D_{\ell -1}$. Since $P_\ell (s,t) = \frac{1}{2n} \sum_{i=1}^n (e^{sQ_i + tW_i} + e^{-sQ_i - tW_i})/2$.

So $H(P_\ell (s,t) - \Phi_{D_{\ell -1}}(s,t)) = (1/n) ( \sum_{i} (Q_i^2 - E[Q_i^2]) s^2 + (W_i^2 - E[W_i^2]) t^2) + (
Q_iW_i - E[Q_iW_i]) st$.

Now $Q_i W_i$ is a bounded random variable with maximum value $M^2$. So averaged over $n$ samples, its standard deviation is at most $M^2/\sqrt n$. So with high probability the coefficient of $st$ in this difference is at most $M^2 \log n / \sqrt n$. For $d < n^{1/6}$ this is $ \le M^2 \log n/d^3$. Similarly we bound the coefficients of $s^2$ and $t^2$ giving
$H(P_\ell (s,t)) \preceq 1 + O(M/d)(s^2 +t^2) + O(M/d^3) st $
This implies $H(\Phi_{D_\ell }(s,t)) = H((P_\ell (s,t))^d) \preceq 1 + O(M)(s^2 +t^2) + O(M/d^2) st $. A corresponding lower bound on the coefficient of $st$ follows in the same way.

\end{proof}

We will now prove a tighter bound on the coefficient of $t^2$ in $\Phi_{D_\ell}(t)$.

\begin{lemma}
With high probability for $\ell \le \log_d n -1$,
$ 1 + t^2/2 - O(\ell \log^2 n / \sqrt d) t^2/2 \preceq H(\Phi_{D_\ell}(t)) \preceq 1 + t^2/2 + O(\ell \log^2 n/ \sqrt d) t^2/2 $.
\end{lemma}
\begin{proof}
Again for $\ell \le \log_d n$ lets drop the scaling factors in the computation of $q_\ell$.
Let $a_\ell $ denote the coefficient of $t^2/2$ in $\Phi_{D_{\ell -1}}$ (that is, it is twice the coefficient of $t^2$).
For $\ell \le \log_d n$, if we know that by induction that coefficient of $a_{\ell -1}$ is $d^\ell /n (1 \pm \eps)$ where $\eps = O(({\ell-1} )\log^2 n/ \sqrt d)$ as per lemma statement then in $P_\ell - E[P_\ell ]$, the coefficient of $t^2$ is $(1/n)\sum_i (Q_i^2-E[Q_i^2])$. Since the maximum value of $Q_i$ is at most $(c\log n) ^{\ell -1}$ w.h.p and it is non zero with probability $O(d^\ell /n)$, $E[Q_i^4] \le O(d^\ell /n) (c\log n)^{4\ell-4}$. So the average value $(1/n)\sum_i (Q_i^2-E[Q_i^2])$ is at most $\sqrt{d^\ell /n}(c \log n)^{2\ell}/\sqrt n \le O((c \log n)^{2\ell}/\sqrt{d^\ell}) d^\ell /n $. For $\ell \ge 1$ this is at most $(c \log n)^2/\sqrt d$. In going from $P_\ell $ to $\Phi_{D_\ell }$ the coefficient of $t^2$ gets multiplied by $d$. Thus proves that $a_\ell =  (1\pm \eps \pm O(\log^2 n/\sqrt d))d ^{\ell +1}/n$. By induction this is within $( 1 \pm O(l \log^2 n / \sqrt d))d ^{\ell +1}/n$.
\end{proof}


\section{Proof of theorem \ref{thm:getX}}\label{sec:recovery}

We will prove that
\begin{theorem}\label{thm:getX}
There is an algorithm that correctly recovers $X$ from $XX^\top $ with high probability.
\end{theorem}
Let $G_X$ be the correlation graph of $XX^\top $ where node $i$ and $j$ are connected with edge weight $w_{ij} = X_iX_j^\top $ if and only if $X_iX_j\neq 0$. In order to find $X$ from the correlation matrix $XX^\top $ we follow the idea of joining nodes with higher correlations. In order to do so, we first join every pair $(i,j)$ such that $w_{ij}\neq 0$ to a new hidden node $p$ in the above layer and we say that $i$ and $j$ are identifying points for the node $p$. So at this time the number of points at the above layer is $\Theta(nd)$. We then join other nodes to an already established node in above level if it has high correlation with both identifying nodes.

Let $N_{ij}$ be the set of all common neighbors between node $i$ and $j$ in graph $G_X$. Here we first show that points with the common support have a high chance of joining together.
\begin{lemma}
\label{lem:prob}
Let $X_i$ and $X_j$ be any two rows of a random $d$-sparse matrix $X$ where $|\supp(X_i)\cap \supp(X_j)| =\alpha>0$. Let $\tilde{N}_{ij} \subseteq N_{ij}$ be the set of all nodes $X_k\in N_{ij}$ such that $\supp(X_i)\cap \supp(X_j) \cap \supp(X_k) \neq \emptyset$. Then with high probability we have that:
$$
o(d)=|N_{ij}-\tilde{N}_{ij}|\leq |\tilde{N}_{ij}| = \alpha d (1-o(1))
$$
\end{lemma}

\begin{proof}
For any common non-zero column $t\in\supp(X_i)\cap \supp(X_j)$, $X_{kt}$ is non-zero with probability $\frac{d}{n}$, So with probability at least $\frac{d}{n}$, $\supp(X_i)\cap \supp(X_j) \cap \supp(X_k)\neq \emptyset$. There is still a possibility that node $k$ is not a neighbor of node $i$ or $j$ in graph $G_X$ because the inner product of $X_\ell$ with $X_i$ or $X_j$ might be zero. This probability is however at most $6d^2/n=o(1)$ (the probability of sharing $\beta$ non-zero entry is at most $(d^2/n)^\beta$ so the total probability of sharing more than one extra non-zero entry is $2d^2/n$). All other possibilities such as probability of $|\supp(X_i)\cap \supp(X_j) \cap \supp(X_k)| >1$ are $o(1)$ and negligible. By Bernstein inequality, we get that with high probability, the number of points with the shared support is at least $\alpha d(1-o(1))$.

Now we want to show that $do(1)=|N_{ij}-\tilde{N}_{ij}|\leq |\tilde{N}_{ij}|$. For any random $X_k\in N_{ij}-\tilde{N}_{ij}$, $X_k$ should have one shared non-zero entry with $X_i$ and another one with $X_j$. So the probability of this event is $d^4/n^2$. Using Bernstein' s inequality, can say that since $d=O(n^{1/6})$, the number of such points in $N_{ij}$ is $o(d)$.
\end{proof}

\begin{lemma}
\label{lem:prun}
Let $X_i$ and $X_j$ be any two rows of a random $d$-sparse matrix $X$. For any $X_k\in N_{ij}$, Let $T_k$ be the set of points $X_{k^\top }\in N_{ij}$ such that $|X_kX_{k^\top }|>0$.The following statement is true with high probability. For any $X_k\in N_{ij}$, $|T_k| =(1-o(1))|N_{ij}|$ if and only if $\supp(X_i)\cap \supp(X_j) \subseteq \supp(X_i)\cap \supp(X_j)\cap \supp(X_k) \neq \emptyset$.
\end{lemma}
\begin{proof}
By lemma \ref{lem:prob} we know that if $\supp(X_i)\cap \supp(X_j) \subseteq \supp(X_i)\cap \supp(X_j)\cap \supp(X_k) \neq \emptyset$ then the number of neighbors of $X_k$ is at least $\alpha d(1-o(1))$ because $(1-o(1))$ of all points that share a non-zero position with $X_i$ and $X_j$, also share a non-zero position with $X_k$. For proving the other direction, assume that there exists a shared non-zero position between $X_i$ and $X_j$ that is zero at $X_k$. Then we know that $X_k$ will not be a neighbor to at least $d(1-o(1))$ of nodes in $N_{ij}$ which proves the lemma statement.
\end{proof}

\begin{proof}(of theorem \ref{thm:getX})
By lemma \ref{lem:prun} We keep all the nodes in $N_{ij}$ that have at least $|N_{ij}|(1-o(1))$ neighbors in $N_{ij}$. Now if the number of shared non-zero entry between $X_i$ and $X_j$ is exactly one, then w.h.p. the number of remaining points will be at least $d/4$. However, if $X_i$ and $X_j$ have more than one shared non-zero entry, then w.h.p. the number of remaining points is at most $O(\log n)$. Therefore, if we discard the hidden nodes in above layer with remaining points less than $d/4$, all the available hidden nodes in the above layer have at least $d/4$ nodes and their identifying points share exactly one non-zero position. since we initially had $nd$ hidden nodes in the above layer, with high probability in the remaining hidden nodes, for each column $k$ of $X$, there exist a hidden node with identifying pairs $X_i$ and $X_j$ that are non-zero simultaneously just at $k$th entry. Now for each hidden node $p$ with identifying pair $(i,j)$ such that $X_i$ and $X_j$ are non-zero at $g$th entry, we look at all $X_k\notin N_{ij}$. It is clear that if $X_k$ is non-zero at $g$th entry then $X_k$ is neighbor with $(1-o(1))$ fraction of nodes in $N_{ij}$ and otherwise $X_k$ is neighbor with at most $o(1)$ fraction of nodes in $N_{ij}$. Therefore, at the end of pruning step, with high probability, for any column $g$ in $X$ there exists a node $p$ in the above layer such that each node $k$ at the below layer is connected to $p$ at the above layer if and only if $X_{kp}\neq 0$.

Now for a column $i$, let $C$ be the set of nodes that have a non-zero value at column $i$. We have already shown that we can find this set with high probability. However, for each $X_j \in C$ we do not know the sign of $X_{ji}$. In order to do so, we first pick a random node in $C$, say $X_j$ and pick a random sign for $X_{ij}$. Then for any $X_k\in C$ we set $X^{(1)}_{ki} = \sign(X_jX_k^\top )i$. Now, we know that the number of nodes with the wrong sign will be at most $o(1)$. So now by setting the sign of each node by taking the majority of suggested signs by its neighbors $X^{(2)}_{ki} = \sign( \sum_{X_j \in C} \sign(X_j X_k^\top ))$, with high probability we can predict the correct sign.

Note that in fact each row of $X$ is not just a sparse sign matrix and with a small probability, it can have higher absolute values (because it is in fact a sum of $d$ 1-sparse random vectors). Finding the magnitude of $X_{ij}$ is not difficult because $1-o(1)$ of its neighbors in $C$ have magnitude 1 so by looking at the correlation between $X_{i}$ and its neighbors, we can find the magnitude of $X_{ij}$ with high probability and repeat the same process for other nodes recursively to get all the values.

\end{proof}

\end{document}